\documentclass[a4paper,fleqn]{cas-dc}

\usepackage[numbers]{natbib}
\usepackage{pifont}
\usepackage{subcaption}
\usepackage{amssymb,amsmath,amsthm}

\newtheorem{proposition}{Proposition}

\usepackage{xcolor}

\usepackage{mathtools, cuted} 

\def\tsc#1{\csdef{#1}{\textsc{\lowercase{#1}}\xspace}}
\tsc{WGM}
\tsc{QE}
\tsc{EP}
\tsc{PMS}
\tsc{BEC}
\tsc{DE}

\begin{document}
\let\WriteBookmarks\relax
\def\floatpagepagefraction{1}
\def\textpagefraction{.001}
\definecolor{mygray}{gray}{0.6}

\shorttitle{Enforcing class separability in metric learning}

\shortauthors{M. Méndez-Ruiz et~al.}

\title [mode = title]{\textbf{\textcolor{mygray}{SuSana Distancia is all you need:}} \\ Enforcing class separability in metric learning via two novel \\ distance-based loss functions for few-shot image classification}        


\author[addr1]{Mauricio Mendez-Ruiz}

\author[addr2]{Jorge Gonzalez-Zapata}

\author[addr2]{Ivan Reyes-Amezcua}

\author[addr1]{Daniel Flores-Araiza}

\author[addr1]{Francisco Lopez-Tiro}

\address[addr1]{Tecnologico de Monterrey, School of Engineering and Sciences, Jalisco, 45138, Mexico}

\author[addr2]{Andres Mendez-Vazquez}\cormark[1]

\address[addr2]{Centro de Investigación y de Estudios Avanzados, Computer Sciences Department, Guadalajara, Mexico}
\cortext[cor]{C. author: gilberto.ochoa@tec.mx, andres.mendez@cinvestav.mx}

\author[addr1]{Gilberto Ochoa-Ruiz}


\begin{abstract}
Few-shot learning is a challenging area of research that aims to learn new concepts with only a few labeled samples of data. Recent works based on metric-learning approaches leverage the meta-learning  approach,  which is encompassed by episodic tasks that make use a support (training) and query set (test) with the objective of learning a similarity comparison metric between those sets. Due to the lack of data, the learning process of the embedding network becomes an important part of the few-shot task.
In this work, we propose two different loss functions which consider the importance of the embedding vectors by looking at the intra-class and inter-class distance between the few data. The first loss function is the Proto-Triplet Loss, which is based on the original triplet loss with the modifications needed to better work on few-shot scenarios. The second loss function, which we dub ICNN loss is based on an inter and intra class nearest neighbors score, which help us to assess the quality of embeddings obtained from the trained network. Our results, obtained from a extensive experimental setup show a significant improvement in accuracy in the miniImagenNet benchmark compared to other metric-based few-shot learning methods by a margin of $2\%$, demonstrating the capability of these loss functions to allow the network to generalize better to previously unseen classes. In our experiments, we demonstrate competitive generalization capabilities to other domains, such as the Caltech CUB, Dogs and Cars datasets compared with the state of the art.
\end{abstract}


\begin{keywords}
Few-Shot Learning \sep Computer Vision \sep Deep Learning \sep Characterization \sep Image Generation 
\end{keywords}

\maketitle

\section{Introduction}

Despite recent advances in deep learning research in various computer vision and NLP tasks \cite{torfi2020natural}, it remains a challenge for the standard supervised learning setting to achieve satisfactory results when learning from just a small amount of labeled data \cite{song2022comprehensive}. Current deep learning algorithms tend to overfit when they are given a small dataset for training, limiting their generalization capabilities \cite{kawaguchi2017generalization}. Moreover, there are many problem domains where obtaining labeled data can be difficult to obtain or it entails a lot of manual work to get the data with the corresponding annotations or ground truth. This represents a major problem in many real-world applications (i.e.,  medicine) as some of the instances we are interested in are very rare \cite{piccialli2021survey}.

Therefore in recent years the Few-shot learning (FSL) paradigm has been proposed (\cite{siamesenetworks, prototypicalNets2017, matchingNetworks2016, sung2017relationNet, maml, reptile}) as a way to deal with this data scarcity problem. For classification tasks, the main goal of such methods is to categorize previously unseen data into a set of new classes, given only a small amount of labeled instances per class. The main challenge for FSL is to apply a fine-tuning process to an existing embedding network to adapt it to new classes; as a matter of fact, the main issue is that such process can easily lead to overfitting, due to the few labeled samples available for each class. 


\begin{figure*}
\centering
\includegraphics[width=.7\textwidth]{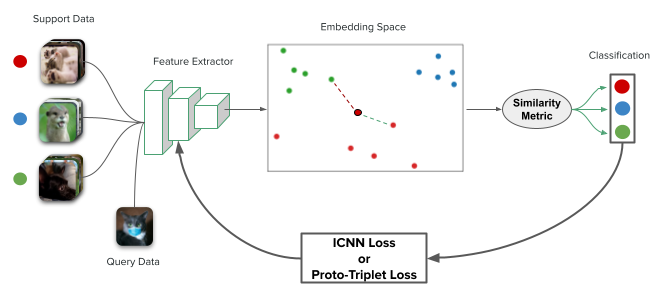}
\caption{In our work we propose to use two loss functions that work well for metric meta-learning approaches for few-shot classification. We optimize an embedding network based on the error calculated by the ICNN Loss or the proto-triplet loss. Both losses aim to increase the inter-class distance and decrease the intra-class distance between samples of different classes.} \label{fig:model}
\end{figure*}

\begin{figure}
\centering
\begin{subfigure}{.47\textwidth}
  \centering
  \includegraphics[width=.75\textwidth]{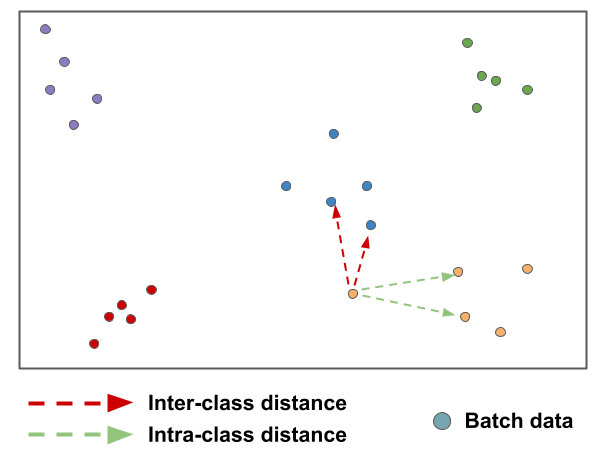}
  \caption{The ICNN Loss is based on a score given by the intra and inter class distance from the batch points. For each data point, we measure its Inter and Intra Class Nearest Neighbors score and optimize the embedding network based on the quality of the generated features.} 
  \label{fig:icnn_loss}
\end{subfigure}%
\hfill
\begin{subfigure}{.47\textwidth}
  \centering
  \includegraphics[width=.70\textwidth]{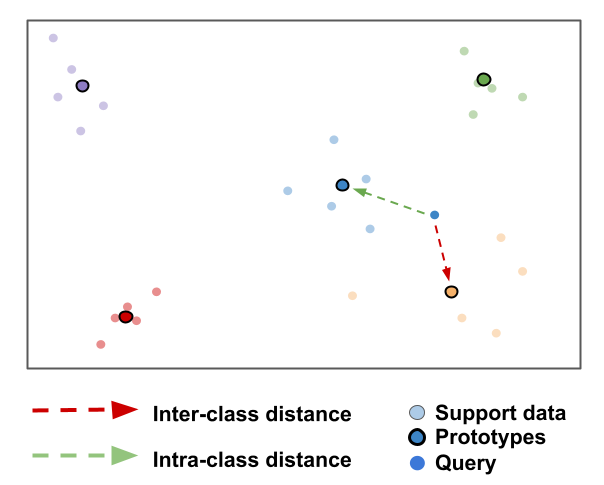}
  \caption{The Proto-Triplet  loss takes as anchor a query point, the prototype from the same class as positive point and the nearest prototype of different class as the negative point.} 
  \label{fig:proto-triplets_loss}
\end{subfigure}
\caption{ Comparison between the two distances.}
\label{fig:test}
\end{figure}

Currently, two main approaches to FSL exist: the first one is based on meta-learning methods (\cite{maml, learningtolearn, learningtooptimize, l2lgradient}), where the basic idea is to learn from diverse tasks and datasets and adapt the learned model to new datasets. A second approach is based on metric-learning methods (\cite{distanceMetricLearning, siamesenetworks}), where the objective is to learn a pairwise similarity metric such that a score (given by some distance) is high for similar samples, whilst dissimilar samples get a low score. Subsequently, these metric learning methods undertook a hybrid approach, as they started to adopt the meta-learning policy to learn across tasks (\cite{prototypicalNets2017, matchingNetworks2016, sung2017relationNet}). The main objective of these methods is to learn an effective embedding network in order to extract useful features from a given task and discriminate on the classes which we are trying to predict. From this basic learning setting, many extensions have been introduced to improve the performance of metric learning methods. Some of these works focus on pre-training the embedding network (\cite{ssl}), others introduce task attention modules (\cite{chen2020multiscale, categoryTraversal, principalCharacteristics}), whereas others seek to optimize the embeddings (\cite{convexoptimization}) and yet others employ a variety of loss functions (\cite{principalCharacteristics}). 
However, only a few methods have explored mechanisms for enforcing class separability via a custom loss function during training.

Our approach introduced herein can be seen as a hybrid between some of the above-mentioned approaches, as we attempt to improve jointly the embeddings and we also investigate the impact of two loss functions for FSL-based classification tasks.

More specifically, we explore the use of two different loss functions based on the concepts of inter-class and intra-class distance. The first one is the well-known proto-triplet loss, which to the best of our knowledge has not been explored in the context of few shot learning; the second, which we dubbed ICNN loss (based on the Inter and Intra Class Nearest Neighbor score first introduced in  \cite{tesis_ivan}) represents a novel loss function that evaluates the quality of the features obtained based on the inter/intra class, the variance and the class ratio from nearest neighbors.

The basic working principle of our proposed approach is shown in Figure \ref{fig:model}. As it can be observed in the figure, the used loss functions allow us to jointly optimize the embedding network and learn more discriminative features across tasks. In our experiments, we obtained an accuracy performance of 61.32\% and 79.93\% in the 5-way 1-shot and 5-way 5-shot settings respectively using the proto-triplet loss, and 60.79\% and 80.41\% using our novel ICNN Loss on the MiniImagenet dataset which represents an improvement of about 2\% compared to the state of the art. We further test our methods on the CUB, Caltech, Stanford Dogs and Stanford Cars datasets to assess the generalization capabilities of the proposed framework, obtaining satisfactory results.

The ICNN Loss and Proto-Triplet Loss are two methods used for optimizing an embedding network based on the quality of the generated features. As shown in Figure \ref{fig:test}, the ICNN Loss measures the intra and inter class distances of batch points and assigns each data point an Inter and Intra Class Nearest Neighbors score. The Proto-Triplet Loss, on the other hand, is based on triplet loss and calculates the similarity between points using prototypes of different classes as anchor, positive, and negative points.

The rest of the paper is organized as follows: In Section 2, we describe the related work in the few-shot learning problem. In Section 3 we describe the proposed loss functions. We first explain the Proto-triplet loss function, and then we explain the rationale for the ICNN Loss function and its derivation. Then we discuss some of the design choices we needed to take for the ICNN loss. In Section 4 we detail the experimental setup used for the implementation of the models. In Section 5 we show the results obtained and discuss about its performance. Finally, in Section 6 we present our conclusions and discuss future work.

\section{Related Work and Motivation}

\subsection{Deep  Metric Learning}

The goal of metric learning is to learn a similarity function from the data. More specifically, it aims to learn feature embeddings in a way that reduces the distance between embeddings corresponding to instances of the same class (intra-class distance) and increases the distance between embeddings corresponding to instances of different classes (inter-class distance). The vast majority of  deep metric learning (DML) approaches make use of an embedding network to learn the discriminating features that will be exploited for computing the similarity metric. Below, we review some of the more relevant deep metric learning methods in the recent literature.

One of the fundamental methods in DML is the Siamese Networks (\cite{siamesenetworks}), which is a symmetric neural network architecture that consists of two sub-networks, both sharing the same parameters. These networks learn their parameters by calculating a distance metric between the feature embeddings of each sub-network,  which are fed a different input. The loss function used in Siamese Networks is the contrastive loss or pairwise ranking loss, which seeks for the distance of samples from the same class to be small and from different classes to be large.

The second important metric learning method is the Triplet Network (\cite{triplets}), whose underlying architecture can also be described by a symmetric neural network. However, this method makes use of three identical sub-networks sharing the same parameters. The input of the three sub-networks is encompassed by three different images: The first one is the anchor (the baseline image), the second is the positive sample (an instance that belongs to the same class as the anchor), and the third is the negative sample (an instance that belongs to a different class than the anchor). This network uses the triplet loss to learn discriminative feature embeddings, and it works by ensuring that the anchor image is close to the positive images and far away from the negative images.

These metric learning methods have been widely used for different purposes, such as image retrieval (\cite{wang2014learning}), face recognition (\cite{triplets, taigman2014deepface, ddml_face}), person re-identification (\cite{xiao2017joint}), video surveillance (\cite{HUANG2018104videosurvilliance}), three-dimensional modelling (\cite{dai2017deep}), signature verification (\cite{signature}), medical image analysis (\cite{annarumma2017deep}), text understanding (\cite{Mueller2016SiameseRA, siameseSemanticPattern}), among other problems.

\subsection{Meta-learning for Few-shot learning}

As deep learning began to yield superior results in many machine learning problems, some authors proposed to use meta-learning policies in order to optimize deep models \cite{learningtolearn, learningtooptimize, l2lgradient}. The meta-learning policy refers to an approach geared towards learning across tasks, followed by a fine-tuning process to adapt the model to new tasks, instead of a learning setting based on samples. The meta-learning objective aims to learn the parameters $\theta$ that minimize the loss across all given tasks.

FSL represents the perfect setting for testing meta-learning algorithms because of the few-labeled data given to each task. \cite{FSLsurvey}. The meta-learning strategy for tackling FSL problems has been traditionally divided into two stages: a meta-train and a meta-test phase. The meta-learning setup consists of episodic tasks, which can be seen as batches in traditional deep learning. A few-shot $K$-way $C$-shot image classification task is given $K$ classes and $C$ images per class. The task-specific dataset can be formulated as $D = \{D_{train}, D_{test}\}$, where $D_{train}=\{(X_i, y_i)\}_{i=1}^{N_{train}}$ denotes the classes reserved for the training phase and $D_{test}=\{(X_i, y_i)\}_{i=1}^{N_{test}}$ denotes the classes reserved for testing.

For each meta-train task $T$, $K$ class labels are randomly chosen from $D_{train}$ to form a support set and a query set. The support set, denoted by $S$, contains $K \times C$ samples ($K$-way $C$-shot) and the query set, denoted by $Q$, contains $n$ number of randomly chosen samples from the $K$ classes.
The training phase uses an episodic mechanism, in which each episode $E$ is loaded with a new random task taken from the training data. For the meta-test phase, the model is tested with a new task $T$ constructed with classes that were not seen during the meta-training stage.

Few-shot learning methods can be categorized based on what a given model seeks to meta-learn. Some approaches consist on having a base-learner and a meta-learner, where the meta-learner parameters are optimized by gradually learning across tasks to facilitate the fast learning of the base-learner for each specific task. Model-Agnostic Meta-Learning (MAML) (\cite{maml}) belongs to this class and its main idea is to search for a good parameter initialization such that the base learner can rapidly generalize with this initialization. Then, REPTILE (\cite{reptile}) incorporates an $L_2$ loss to simplify the computation of MAML. Further on, LEO (\cite{leo}) was proposed as a network to learn low dimensional latent embeddings of the model. 

On the other hand, Meta-SGD (\cite{meta-sgd}) also learns the base learner update direction and learning rate on the meta-learning process. Meta-Learner LSTM (\cite{optimizationfewshot}) proposes to fine-tune the base learner by a LSTM-based meta-learner, which takes as input the loss and gradient of the base learner with respect to each support sample. Other approaches seek to learn the similarity metric that is expected to be transferable across different tasks.

\subsection{Metric meta-learning for few-shot learning}

There exists an entire branch of meta-learning approaches that aim to solve the few-shot learning problem by leveraging some of the core ideas of metric learning. These approaches adopt the meta-learning setup to learn the similarity metric expected to generalize across different tasks. Certain baseline methods have achieved important milestones for few-shot learning, such as Prototypical Networks (\cite{prototypicalNets2017}), Matching Networks (\cite{matchingNetworks2016}) and Relation Networks (\cite{sung2017relationNet}). Prototypical Networks is the model we will be using as a basis in this work, as it also functions by taking the center of support samples embeddings from each class to create class prototypes. 

These models use a distance metric (typically the euclidean distance) to predict the probability of belonging to a class for each query sample. The Matching Networks predict the probability of a query sample by measuring the cosine similarity between the query embedding and each support sample embedding. On the other hand, the Relation Networks adopt a learnable CNN as the pairwise similarity metric, which takes the concatenation of feature maps of the support sample and the query sample as input and outputs the relation score. These three methods can be considered the base metric learning approaches for few-shot learning.

Further on, some recent works have focused on introducing task attention modules (\cite{chen2020multiscale, categoryTraversal, principalCharacteristics}), whereas others try to optimize the embeddings (\cite{convexoptimization}) and others add a second term to the loss function (\cite{principalCharacteristics}). Up until now, there has been a lack of research for loss functions that work for the problem of few-shot learning tackled from a metric-learning perspective. Although these methods utilize
deep networks to extract discriminant deep features, they do not take full advantage of
the relationship among the input samples. Hence, we are motivated to explore strategies to improve the feature embedding in terms of their efficiency to be transferable to
handle unseen class samples and their generality for few-shot classification.

Our work, which is based on prototypical networks, seeks to contribute alleviate this problem, by introducing the concept of distance in the loss function. Our proposed model belongs to these meta-learning approaches based on metric learning, by adopting the idea of the intra-class and inter-class variance into account for the construction of two different loss functions, which will help us to better optimize an embedding network to obtain more discriminant feature vectors.  We aim to meta-learn a feature embedding that performs well, not only in the training classes but more importantly, in the novel classes. 
Specifically, the feature embedding should map similar samples close to one another and dissimilar ones far apart. This is well-aligned with the philosophy of triplet-like learning. However, the general triplet network only interacts with
a single negative sample per update, while few-shot classification requires a comparison with multiple query samples, typically of different classes.

For this, we make use of two novel loss functions, briefly introduced above: a modified proto-triplet loss and inter-intra class nearest neighbors score loss for training a Convolutional Neural Network (CNN) as an embedding function. These loss functions optimize the embedding space across the meta-learning tasks. Our experimental results show the proposed losses yield an increase in accuracy in the miniImagenNet benchmark compared to other metric-based few-shot learning methods by a margin of $2\%$. We also demonstrate the capability of these loss functions to produce models able to generalize better to previously unseen classes, as demonstrated by experiments in the CUB-200, Caltech-101 and the Stanford Dogs and Cars datasets.


The main contributions of this paper are two loss functions based on contrastive learning approaches. The first loss function is inspired by the well-known triplet loss with adjustments for the few-shot learning setting. The second loss function is a novel method that calculates the error based on the quality of feature embeddings obtained.

\section{DML-based loss functions for FSL}


\subsection{Proto-Triplet Loss}

The traditional Triplet Loss (\cite{triplets}) is widely used for metric learning, where the objective is to train a learner by using a similarity comparison metric between the sampled data. The loss function aims at pulling similar samples close to each other while pushing away the samples of different classes far away. This is the classic paradigm of clusters for good classification. The objective of the triplets is to create a manifold, where classes live on. Thus, the triplets are used to learn why classes are different and why are similar on the manifold.

Our proposed Proto-Triplet Loss is based on the original Triplet Loss with the difference that we are using the data points obtained as prototypes for calculating the loss, making it more suitable for the few-shot learning setting. 
The formula for the proto-triplet loss is similar to the original triplet loss:
{
\begin{equation}
    L(X_a, X_p, X_n) = \left [  \left \| f_a - f_p \right \|^2 - \left \| f_a - f_n \right \|^2 + \alpha \right ]_+ ,
\end{equation}}

where $\left [ \right ]_+$ is the max between 0 and the result, also known as hinge loss. $f_a$ is the embedding from the query, $f_p$ is the prototype from the same class as the query and $f_n$ is the nearest prototype of a different class as the query. $\alpha$ is a margin hyperparameter to be enforced between the negative and positive pairs. 
The goal of this function is to keep the distance between the query and the prototype of the same class smaller than the distance between the query and the prototype from a different class.

Furthermore, inspired in (\cite{ktuplets}), we propose to use the Proto-Triplet Loss with $K$ reference points to use more than one negative sample. In the meta-learning setting, we compare an instance with multiple classes, so our loss function should be more effective by performing a comparison with more prototypes of different classes.

For the negative samples, we take the $K$ nearest prototypes of different classes to form the triplets. The modified function is formulated as:

{\small
\begin{equation}
    L(X_a, X_p, X_n) = \frac{1}{K} \sum_{i \in U} \left [  \left \| f_a - f_p \right \|^2 - \left \| f_a - f_n \right \|^2 + \alpha \right ]_+ ,
\end{equation}}
where $U$ is the set of triplets made by the $K$ negative prototypes. The rest of the terms in the formula are the same as stated in the previous equation.

\subsection{ICNN Loss}
Given the basic idea of class separation and high density per class, the Inter and Intra Class Nearest Neighbors Score (ICNN Score) (\cite{tesis_ivan}) is proposed as a measure to remove noisy features and improve the performance of manifold-based algorithms for supervised dimensionality reduction, aiding the feature selection by a subset evaluation. The ICNN score is the result of searching for better methods of feature selection than frameworks based on things as the Maximal Information Coefficient \cite{make4010007}. For this, the ICNN is based on two main concepts: the inter-class and the intra-class distances. The former describes the distance between data points of different classes, while the latter refers to the distance between data points of the same class.

The ICNN assigns a score by measuring the distance and variance of the inter and intra $k$-nearest neighbors of each instance in the data. The equation to measure the data points and assign them a score is represented by three different terms: $\lambda$, $\omega$ and $\gamma$. The ICNN formula is given as follows:
\begin{equation}
    ICNN(X) = \frac{1}{\left | X \right |} \sum_{x_i \in X} \lambda(x_i)^\frac{1}{p} \omega(x_i)^\frac{1}{q} \gamma(x_i)^\frac{1}{r},
\end{equation}
where $p$, $q$ and $r$ are control constants.

\section{Designing the ICNNScore}
In order to build and understand the Inter/Intra Class Nearest Neighbors Distance Score (ICNNScore), it is necessary to define the concepts of  inter-class distance and intra-class distance \cite{james2013introduction}. The main idea for inter-class distance is that different classes should be distant or optimal with respect to a distance $d$ (Eq. \ref{GeneralInterClassDistance}),
\begin{equation}
D^*\left( C_i, C_j \right) = \sup_d
\frac{1}{n_i n_j} \sum_{x_j, x_k \in C_i}  d\left( x_j,  x_k \right), \label{GeneralInterClassDistance}
\end{equation}
which basically represents an average of the maximal distance matrix between elements of the two classes $i$ and $j$. 
In a similar way, the optimal intra-class distance is defined as
\begin{equation}
I^* \left( l_i \right) = \inf_d \frac{1}{n_i} \sum_{k_1=1}^{n_i} \sum_{k_2=k_1+1}^{n_i} d\left( x_{i, k_1},  x_{j,k_2} \right).
\end{equation}
These quantities represent the distance between different classes and the density of those classes, respectively. Although these equations represent the basic ideas of inter- and intra-class distances, they are useful to develop a score that can be transformed into a loss function. For this, we propose  a function that minimizes penalization to the neighborhood class elements to a particular sample $x_i \in C_j$. In addition, we also define a function that  maximizes the penalization of elements $x_k \notin C_j$ to such sample $x_i$ based on their distance proximity.  In the ideal case with respect to $x_i$ at its neighborhood, the distances of elements in the same class are close to zero. In addition,  elements in the same neighborhood, but in different classes than $x_i$, their distances to it are as large as possible (Figure \ref{fig:NeighborhoodICNNScore}). The former bound is easy to control and has tendency to zero. However, the latter bound can be problematic given the different distances that can be obtained. Thus, instead of having any large random possible distances, we force the  bound of these distances to a known constant.
\begin{figure}
\centering
\includegraphics[width=.3\textwidth]{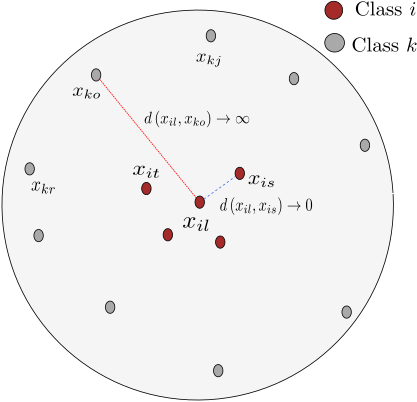}
\caption{The ideal case for the ICNNScore} \label{fig:NeighborhoodICNNScore}
\end{figure}
For this purpose, we define the following neighborhoods: $N_{x_i}$ the $k$-nearest neighbors of  $x_i$ at the same class, and $N_{\tilde{x}_i}$ the $k$-nearest neighbors of $x_i$ at different classes. Using these sets, it is possible to define the following function  based on a differentiable distance $d\left( x_i, x_j \right)$,
\begin{align}
 \alpha \left( x_i \right) = \max \left\{ d\left( x_i, p \right) \vert p \in N_{x_i} \cup N_{\tilde{x}_i} \right\} 
\end{align}
We also have the distance function,
\begin{equation}
  \theta \left( x_i \right) = \min \left\{ d\left( x_i, p \right) \vert p \in N_{x_i} \cup N_{\tilde{x}_i} \right\} 
\end{equation}
Using these functions, we define a normalization of the distances for all the elements in the neighborhood of $x_i$:
\begin{equation}
 h \left( x_i, p \right) = \frac{d\left( x_i, p \right) - \theta\left( x_i \right) }{ \alpha\left( x_i \right) - \theta\left( x_i \right)},
\end{equation}
which is a min-max normalization that enforces the desired bounds in the interval $[0,1]$. Now, we define a penalization $\lambda$ function for defining the concept of "far" or "near" to a specific $x_i$:
\begin{align}
    \lambda_{\tilde{x}} \left( x_i \right) & = \sum_{p\in N_{\tilde{x}_i}} h \left( x_i, p \right) \label{TotalNormDistanceNotClass}, \\
    \lambda_{x} \left( x_i \right) & = \sum_{p\in N_{x_i}} \left[ 1-h\left( x_i, p \right) \right] \label{TotalNormDistanceClass}
\end{align}
The first equation (Eq. \ref{TotalNormDistanceNotClass}) represents the normalized penalization of the elements in a certain neighborhood of elements in different classes than the class of $x_i$, and the second equation represents a similarity to elements at $N_{x_i}$. Therefore, we have
\begin{align}
    \lambda\left( x_i \right) = \frac{\lambda_{\tilde{x}} \left( x_i \right) + \lambda_{x} \left( x_i \right) }{\left|N_{\tilde{x}_i} \cup N_{x_i}  \right|}
    \label{Lambda:BeforeFixing}
\end{align}
Thus, if we want to maximize this function, the average $\lambda_{\tilde{x}}/N$, with $N=\left|N_{\tilde{x}_i} \cup N_{x_i}  \right|$, needs to be maximized. Thus, the distances of the elements not belonging to the class of $x_i$ are maximized to it. In the case of $\lambda_{x}/N$, the similarities are maximized toward one, thus forcing the elements belonging to the class of $x_i$ to get near to it or to zero. However, when going through the battery of experiments a change is proposed on the original $\lambda$ function (Eq. \ref{Lambda:BeforeFixing}) to improve results for the few-shot learning setup,
\begin{equation}
    \lambda \left( x_i \right) = \frac{\lambda_{\tilde{x}} \left( x_i \right)}{\left|N_{\tilde{x}_i}\right|} + \frac{\lambda_{x} \left( x_i \right)}{\left| N_{x_i}  \right|}.
\end{equation}
Therefore, we have that each of $\lambda's$ $x$ functions finish as
 \begin{align}
     \frac{\lambda_{\tilde{x}} \left( x_i \right)}{\left|N_{\tilde{x}_i}\right|} & \thickapprox E_{h}  \left[ h\left( x_i , p \right) \vert p \in N_{\tilde{x}_i} \right],  \nonumber \\
         \frac{\lambda_{x} \left( x_i \right)}{\left| N_{x_i}  \right|} & \thickapprox 1 -  E_{h} \left[ h\left( x_i , p \right) \vert p \in N_{x_i} \right]  \nonumber.
 \end{align} \noindent These expected values depend on the original distribution of the $d\left( x_i, p \right)$ given that we have a uniform a priori, and the likelihood of the defined distance. Now, if the $\lambda$ function is maximized with respect to the representation of $x_i$ and $p$, two phenomena occur: First, the inter-distance on $N_{x_i}$ is minimized, and second, the intra-distance on $N_{\tilde{x}_i}$ is maximized. This basically corresponds to highly dense and compact class clusters with large distances between them, our ideal case for classification. Therefore, integrating this $\lambda$ function to a cost function will enforce these properties in the generated embedding during the learning process of a neural network. 
 
\begin{figure}
\centering
\includegraphics[width=0.45\textwidth]{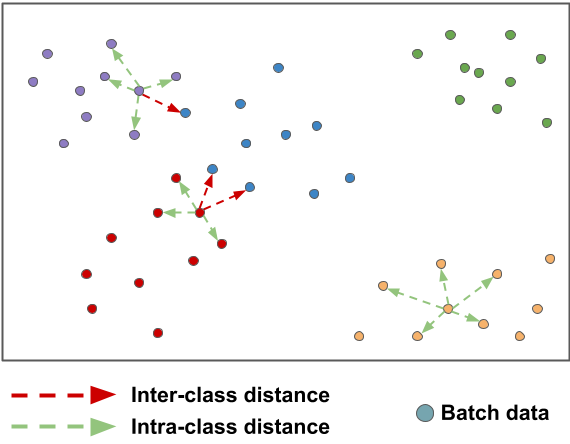}
\caption{Lambda ($\lambda$) is the function that penalizes the neighbors of $X_i$ with the same class based on how distant they are, and the neighbors of different classes based on how close they are.} \label{fig:lambda}
\end{figure}

\begin{proposition}
\label{Proposition1}
Given the $\lambda$ function, and a smooth embedding representation for the samples $x_i$ based on weights $\mathcal{F}$. If $\lambda$ is maximized by using a gradient method on those weights, the embedding will change to increase the inter-class distance and minimize the intra-class distance at the classes.
\end{proposition}
\begin{proof}
Appendix \ref{Proof:Proposition1}
\end{proof}
This proposition allows to say that a cost function based on the ICNN  helps to separate the class clusters when training on data. Next, given that we have  $h\left(x_i \right)\leq 1$, if we add overall terms $i$, we finish with
\begin{align}
    \sum_{p\in N_{\tilde{x}_i}} h\left(x_i, p \right) & \leq \left| N_{\tilde{x_i}} \right| \longrightarrow \frac{\lambda_{\tilde{x}}\left( x_i \right)}{\left| N_{\tilde{x_i}} \right|} \leq 1
\end{align}
In a similar way, $0 \leq \frac{\lambda_x \left( x_i \right) }{ \left| N_{x_i} \right|}\leq 1$ given that also $h\left(x_i , p\right) \leq 1$ for $p \in N_{\tilde{x}} \cup N_{x}$. In this way, $ 0 \leq \lambda \left( x_i \right) \leq 2$.

Now, we would like to avoid large changes while learning the distances between elements in the neighborhood of $x_i$. For this, we introduce a penalization function over the terms $h$ that minimizes the possible variance (Eq. \ref{VarianceDistanceMinimization}).
\begin{equation}
    \omega \left( x_i \right) = \alpha - \left[ Var\left( \lambda_{\tilde{x}} \left( x_i \right)  \right) +  Var\left( \lambda_{x}\left( x_i \right)  \right) \right]
    \label{VarianceDistanceMinimization}
\end{equation}
The variance of this function depends on the original distribution of values $d\left( x_i, p \right)$. 

\begin{proposition}
\label{Proposition2}
Given the $\omega$ function, and the smooth embedding representation based on weights $\mathcal{F}$. If the loss function is minimized, the $\omega$ is maximized by using the gradient method on those weights, then the inter and intra-class variances can be minimized by selecting an adequate $alpha$.
\end{proposition}
\begin{proof}
Appendix \ref{Proof:Proposition2}
\end{proof}

Thus, optimizing a loss function based on the ICNN helps to minimize the variance of the learned manifold or in other words it favors dense clusters. In order to obtain the bounds for $\omega$, we can use a classic definition of variance,
\begin{align*}
    Var\left( \lambda_{\tilde{x}} \left( x \right) \right)
    & = E \left( \lambda_{\tilde{x}}^2 \left( x \right) \right) - E^2\left[ \lambda_{\tilde{x}} \left( x \right) \right]
\end{align*}

\noindent
Thus, 
\begin{align*}
    E\left[ \lambda_{\tilde{x}}\left(x \right) \right] & \thickapprox \frac{1}{\left|X \right|} \sum_{x\in X} \lambda_{\tilde{x}} \left( x \right) \\
    & \thickapprox  \frac{1}{\left|X \right|} \sum_{x\in X} \sum_{p\in N_{\tilde{x}}} h \left( x, p \right) 
\end{align*}
\noindent In this way, we have that the variance is bounded as
\begin{align*}
    Var\left( \lambda_{\tilde{x}} \left( x \right) \right) \leq & 
     \frac{1}{\left|X \right|} \sum_{x\in X} \left( \sum_{p\in N_{\tilde{x}} }  h \left( x, p \right) \right)^2 -...\\
     & \left( \max_{x_i} \left| N_{\tilde{x}_i} \right| \right)^2 \\
     \leq &   \frac{1}{\left|X \right|} \sum_{x\in X} \left( \sum_{p\in N_{\tilde{x}} }  h \left( x, p \right) \right)^2 \\
     & - \left( \max_{x_i} \left| N_{\tilde{x}} \right| \right)^2 \leq k_1^2,
 \end{align*}
 where $k_1 = \max_{x_i} \left| N_{\tilde{x}} \right|$
 This $k_1$ happens when all the elements at $N_{\tilde{x}}$ have a distance of one. In a similar way, we have $Var\left( \lambda_{x} \left( x \right) \right) \leq k_2^2$ which is also  $k_2 = \max_{x_i} \left| N_{x} \right|$
 In the reverse case, we have that 
the lower bound for both variances is 0.
 In this way, we have that we can use $\omega = k_1^2+k_2^2$ to go force the best case where $x$ has the maximal separation from elements not in the class and minimal separation from elements in the same class. 

 \begin{figure}
\centering
\includegraphics[width=0.45\textwidth]{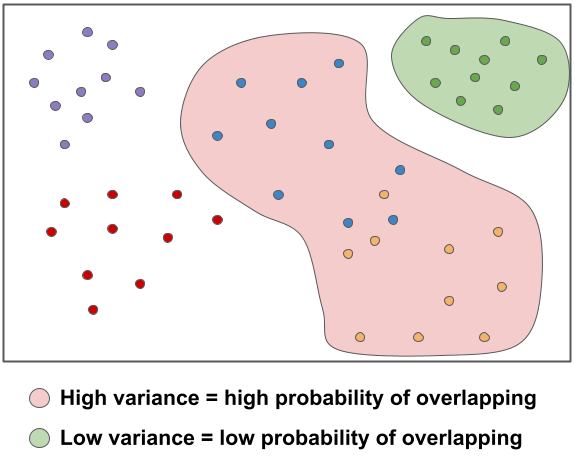}
\caption{Omega ($\omega$) penalizes the distance variance of neighbors. High variance is not desirable as it increases the chance of class overlaps.} \label{fig:omega}
\end{figure}

Finally, we have the last penalization for the ICNNScore which represents how many elements in the neighborhood  $N_{\tilde{x}_i} \cup N_{x_i} $ belong  to the same class as $x_i$,
\begin{equation}
    \gamma \left( x_i \right) = \frac{\left| N_{x_i}  \right|}{\left| N_{x_i}  \right|+\left| N_{\tilde{x_i}}  \right|}
\end{equation}
Therefore, using these three functions, it is possible to build the ICCNScore,
\begin{equation}
    ICNN\left(X\right) = \frac{1}{\left| X \right|}  \sum_{x_i \in X} \lambda\left(x_i\right)^{\frac{1}{p}} \omega\left(x_i\right)^{\frac{1}{q}} \gamma\left(x_i\right)^{\frac{1}{r}}
\end{equation}
In order to use this score as a loss function  we  simply use the negative value of the logarithm function,
\begin{equation}
    ICNNLoss\left( X \right) = - log\left( ICNN\left(X\right) \right)
    \label{ICNNLOSS:def}
\end{equation}
Thus minimization of the ICNNLoss becomes the maximization of each of the elements of the ICNN score.

\section{Experimental Design}


The experiments made on this work were designed to answer the following questions: First, is the Proto-Triplet Loss competitive with other metric-learning state-of-the-art models? Second, is the ICNN Loss competitive with other metric-learning state-of-the-art models? Third, how do our loss functions optimize the feature space to make the features more discriminating?
 
The results obtained in this study indicate that both the Proto-Triplet loss and the ICNN loss are competitive with current state-of-the-art approaches. Furthermore, when these loss functions are combined, they achieve a high level of performance that surpasses other models. Finally, we can see in Figure \ref{fig:umaps} that our loss functions optimize the feature space by providing a clearer intra-class and inter-class separation of the feature embeddings.

Having defined the equation for the ICNN Loss function, we test different scenarios by varying the data points given to the algorithm and combining it with cross entropy loss and Proto-Triplet loss. For training the network, we know to which class the support data and the query data belong, therefore, we can use both sets to assign a score to the current batch. The different options were tested by assigning a different score to the support data, the query data and both the support and query data. For the support data, calculating the ICNN Score is straightforward, as each data point has an assigned score based on its nearest neighbors and is averaged across the set. For the query data, each query is given a score based on the nearest neighbors of the support data points. When using both the support and query data, the data points are combined and considered for calculating the score for each instance. The options for the algorithm are the following:
    \begin{itemize}
        \item Score only on support data
        \item Score on support data and score in query data
        \item Score of support and query data together
    \end{itemize}

We can also opt to make a different choice of data given to the ICNN algorithm when working with the query set. As our method aims at classifying the queries based on the prototypes, we added the option to give a score to the query set by using the prototypes as neighbors. Each query instance is given a score by using the $k$ nearest prototypes.

\section{Materials, Metrics, and Methods}

\subsection{Datasets}
We evaluate our experiments using the MiniImageNet dataset (\cite{matchingNetworks2016}), which is a version of the ImageNet Large Scale Visual Recognition Competition 2012 (\cite{imagenet}). Following the split proposed by \cite{optimizationfewshot}, this version of ImageNet is divided into 64 classes for training, 16 classes for validation and 20 classes for testing, making a total of 100 classes for the meta-learning. Each class contains 600 images to have a total of 60,000 images. This dataset is used as a benchmark to evaluate most of the state-of-the-art few-shot learning methods.

We further test the proposed framework with other well-known datasets for assessing the generalization capabilities of te final model trained in MiniImagnet, namely the Caltech-101, Stanford Dogs, Stanford Cars and CUB datasets. Caltech \cite{griffin2007caltech} is a dataset consisting of 101 widely varied categories, as well as a background category. For each class, there are around 40 and 800 images, while most classes have about 50 images. These classes are randomly split to use 20 classes as the testing set. Stanford Cars \cite{stanford_cars} contains fine-grained images from 196 classes of cars.


\subsection{Evaluation Metrics}
We follow the same method as other metric learning methods for evaluating our results (\cite{prototypicalNets2017, matchingNetworks2016, sung2017relationNet}), reporting the mean accuracy (\%) of 1,000 randomly constructed tasks taken from the testing set along with the 95\% of confidence interval. Each task of the testing phase contains 15 query images per class. 

\subsection{Implementation Details}

We follow the same setting as other few-shot learning models (\cite{matchingNetworks2016, prototypicalNets2017, sung2017relationNet}), making the experiments under the setting of 5-way 1-shot and 5-way 5-shot, and using 15 query images for each class in the task. The input images are resized to 84 $\times$ 84, and then normalized. We construct 100 random tasks for the training phase over 200 epochs. We validated each epoch with another 500 randomly constructed tasks using images from the validation set. Finally, in the testing phase we construct 1000 tasks with images from the testing set. We use the same random seed over all the experiments.

Based on the current state-of-the-art typical setting, we tested two different networks for the feature extractor: a ConvNet and a ResNet-12. Using larger networks makes the model more prone to overfitting, since there is low availability of data. The ConvNet follows the same architecture setting as previous works (\cite{matchingNetworks2016}, \cite{prototypicalNets2017}). This network is composed of 4 layers of convolutional blocks, with each block having a 3 $\times$ 3 convolution with 64 filters, followed by a batch normalization and a ReLU layer. For this network, we used the Adam Optimizer, as in previous works, with an initial learning rate of $1 \times 10^{-3}$ and a step size of 20. The ResNet follows the same architecture as other recent works (\cite{categoryTraversal}, \cite{convexoptimization}). This network is pre-trained with images from the training set, using the Stochastic Gradient Descent (SGD) optimizer with a momentum of 0.9 and a learning rate of 0.1 over 100 epochs with a batch size of 128. After the pre-training, we meta-train the network using the SGD optimizer with a learning rate of $1 \times 10^{-4}$, momentum of 0.9 and a step size of 20.


\subsubsection{Ablation Studies}

The ablation studies were split into three separate experimental setups: experiments with different combinations of the proto-triplets and the ICNN lossess and various tests with combinations of the network used as feature extractor (ConvNet and ResNet-12). For all the experiments, we repeated the design choices using both networks to compare against baselines and recent methods in few-shot image classification. The models using a ResNet-12 greatly improve the accuracy performance compared to those using a ConvNet.

The initial step in our methodology is to conduct experimentation in order to assess the algorithm's efficacy when applied to the individual components in a few-shot task. 
\begin{itemize}
    \item (\textit{i}) ICNN using only the support data. This represents the most basic use of our loss function, since it is used only on a subset of the data available on each training task.
    \item (\textit{ii}) ICNN using the support data and the query data separately with the aim of determining the degree of improvement achieved when utilizing both sets of data, in comparison to utilizing only the support data. The loss for each set is calculated and subsequently combined to determine the overall loss.
\end{itemize}
The next step of our experimentation incorporates the cross-entropy loss function into the pre-existing settings in order to evaluate how the model improves by its inclusion.
\begin{itemize}
    \item (\textit{iii}) Addition of the cross-entropy loss and ICNN using only the support data. This experiment was executed to compare how the model performance improves under the most basic setting.
    \item (\textit{iv}) Addition of the cross-entropy loss and ICNN in support and query data separately. Similarly to the previous settings, we added the cross-entropy loss to evaluate the impact on the performance of the model.
\end{itemize}
In the next phase of experimentation, we introduce prototypes into the task data. The prototypes are computed using both the support and query sets, and subsequently used for the calculations of the ICNN algorithm.
\begin{itemize}
    \item (\textit{v}) ICNN with the prototypes from the support and query data. The rationale behind this approach is that the prototypes serve as the central representation of each class embedding, therefore the ICNN algorithm should be able to provide accurate scores with these mean representations.
    \item (\textit{vi}) Add the cross entropy and the ICNN with the prototypes from the support and query data. Similarly to previous settings, this experiment aims to determine how the model performance is enhanced by adding the cross-entropy loss to the prototypes setting.
\end{itemize}
In the final step of our experimentation, we evaluate the algorithm's performance by calculating the ICNN loss using the complete task data, which is obtained by combining the support and query sets.
\begin{itemize}
    \item (\textit{vii}) ICNN using the full data from each task. Through this approach, our goal is to evaluate the algorithm's performance using the entirety of the data available for each task.
    \item (\textit{viii}) Adding the cross-entropy loss and the ICNN using the full data from each task. Same as with previous settings, we aim to evaluate how the cross-entropy improves the current ICNN setting.
\end{itemize}

\section{Results and Discussion}

\subsection{Standard few-shot learning evaluation}

\setlength{\tabcolsep}{10pt}

\begin{table*}[]
\centering
\begin{tabular}{@{}rccc@{}}
\toprule
\textbf{Model}                                    & \multicolumn{1}{l}{\textbf{Feature Extractor}} & \textbf{1-shot}             & \textbf{5-shot}             \\ \midrule
\multicolumn{1}{r|}{Matching Networks \cite{matchingNetworks2016}}            & ConvNet                                        & 43.56 $\pm$ 0.84\%          & 55.31 $\pm$ 0.73\%          \\
\multicolumn{1}{r|}{Prototypical Networks \cite{prototypicalNets2017}}        & ConvNet                                        & 49.42 $\pm$ 0.78\%          & 68.20 $\pm$ 0.66\%          \\
\multicolumn{1}{r|}{Relation Networks \cite{sung2017relationNet}}            & ConvNet                                        & 50.44 $\pm$ 0.82\%          & 65.32 $\pm$ 0.70\%          \\
\multicolumn{1}{r|}{Baseline* \cite{chen2019closer}}                    & ConvNet                                        & 41.08 $\pm$ 0.70\%          & 54.50 $\pm$ 0.66\%          \\
\multicolumn{1}{r|}{MAML \cite{maml}}                         & ConvNet                                        & 48.70 $\pm$ 1.84\%          & 63.11 $\pm$ 0.92\%          \\
\multicolumn{1}{r|}{Reptile \cite{reptile}}                      & ConvNet                                        & 49.97 $\pm$ 0.32\%          & 65.99 $\pm$ 0.58\%          \\ \midrule
\multicolumn{1}{r|}{Ours Proto-Triplet} & ConvNet   & \textbf{49.82 $\pm$ 0.73\%} & \textbf{68.76 $\pm$ 0.46\%} \\
\multicolumn{1}{r|}{Ours ICNN}                   & ConvNet                                        & \textbf{49.71 $\pm$ 0.78\%} & \textbf{68.66 $\pm$ 0.58\%} \\ \midrule
\multicolumn{1}{r|}{SNAIL \cite{snail}}                        & ResNet-12                                      & 55.71 $\pm$ 0.99\%          & 68.88 $\pm$ 0.92\%          \\
\multicolumn{1}{r|}{DN4 \cite{dn4}}                          & ResNet-12                                      & 54.37 $\pm$ 0.36\%          & 74.44 $\pm$ 0.29\%          \\
\multicolumn{1}{r|}{TADAM \cite{oreshkin2019tadam}}                        & ResNet-12                                      & 58.50\%                     & 76.70\%                     \\
\multicolumn{1}{r|}{K-tuplet Net \cite{ktuplets}}                 & ResNet-12                                      & 58.30 $\pm$ 0.84\%          & 72.37 $\pm$ 0.63\%          \\
\multicolumn{1}{r|}{ProtoNets + CTM \cite{categoryTraversal}}              & ResNet-12                                      & 59.34 $\pm$ 0.55\%          & 77.95 $\pm$ 0.06\%          \\
\multicolumn{1}{r|}{Principal Characteristic Net \cite{principalCharacteristics}} & ResNet-12                                      & 63.29 $\pm$ 0.76\%          & 77.08 $\pm$ 0.68\%          \\ \midrule
\multicolumn{1}{r|}{Ours Proto-Triplet} & ResNet-12 & \textbf{61.32 $\pm$ 0.52\%} & \textbf{79.93 $\pm$ 0.76\%} \\
\multicolumn{1}{r|}{Ours ICNN}                   & ResNet-12                                      & \textbf{60.79 $\pm$ 0.62\%} & \textbf{80.41 $\pm$ 0.47\%} \\ \bottomrule
\end{tabular}
\caption{Test accuracies on miniImagenet in the 5-way setting for both 1-shot and 5-shot.}
\label{tab:comparison-stateoftheart}
\end{table*}

For the MiniImageNet dataset, we evaluate our method under the two most common few-shot learning settings: 5-way 1-shot and 5-way 5-shot. The results are compared against base metric learning methods for few-shot classification (\cite{prototypicalNets2017, matchingNetworks2016, sung2017relationNet}), and against recently proposed methods (\cite{ktuplets, categoryTraversal, principalCharacteristics}).

As detailed in Table \ref{tab:comparison-stateoftheart}, our method outperforms most of the baselines in the 5-way 1-shot setting, getting only lower accuracy than Reptile (\cite{reptile}) and Relation Networks \cite{sung2017relationNet}) when using a ConvNet as Feature extractor, and lower accuracy than Principal Characteristics Net (\cite{principalCharacteristics}) when using a ResNet-12. For the 5-way 5-shot setting, our method outperforms all the baseline methods and the recent ones using ConvNet and ResNet-12 as feature extractors.

When using a ConvNet as the feature extractor, our method achieves an improvement of 8.8\% over the baseline (\cite{chen2019closer}), 5.3\% over the Matching Networks (\cite{matchingNetworks2016}) and 0.4\% over the base prototypical Networks for the 5-way 1-shot setting. For the setting of 5-way 5-shot, we achieve an improvement of 13.4\% over the Matching Networks, 3.4\% over the Relation Networks and 0.5\% over the prototypical networks.

When using a ResNet-12 as the feature extractor, our method achieves an improvement of 5.6\% over SNAIL (\cite{snail}), 3\% over the K-tuplet Net (\cite{ktuplets}) and 2\% over the Category Traversal Module (\cite{categoryTraversal}) for the few-shot setting of 5-way 1-shot. On the setting of 5-way 5-shot, we obtained an improvement of 11.6\% over SNAIL, 8.1\% over the K-tuplets Net and 3.4\% over the Principal Characteristics Net (\cite{principalCharacteristics}).

\subsection{Results of the ablation studies}

\begin{table*}[]
\centering
\begin{tabular}{@{}rcll@{}}
\toprule
\textbf{Model} &
  \multicolumn{1}{l}{\textbf{Backbone}} &
  \multicolumn{1}{c}{\textbf{1-shot}} &
  \multicolumn{1}{c}{\textbf{5-shot}} \\ \midrule
\multicolumn{1}{r|}{$(i)$ ICNN in Support}                                     & ConvNet                     & \multicolumn{1}{c}{41.82} & \multicolumn{1}{c}{55.52} \\
\multicolumn{1}{r|}{$(ii)$ ICNN in Support + Query}                             & ConvNet                     & \multicolumn{1}{c}{42.81} & \multicolumn{1}{c}{55.93} \\
\multicolumn{1}{r|}{$(iii)$  CrossEntropy + ICNN in Support}                      & ConvNet                     & \multicolumn{1}{c}{49.28} & \multicolumn{1}{c}{68.33} \\
\multicolumn{1}{r|}{$(iv)$ CrossEntropy + ICNN in Support \& Query} &
  ConvNet &
  \multicolumn{1}{c}{\textbf{49.71}} &
  \multicolumn{1}{c}{\textbf{68.66}} \\
\multicolumn{1}{r|}{$(v)$ ICNN in Support \& Query(Prototypes)}                 & ConvNet                     & 41.81                     & 55.93                     \\
\multicolumn{1}{r|}{$(vi)$ CrossEntropy + ICNN in Support\&Query(Protos)} &
  ConvNet &
  \multicolumn{1}{c}{48.86} &
  \multicolumn{1}{c}{67.74} \\
\multicolumn{1}{r|}{$(vii)$ Full ICNN}                                           & ConvNet & 45.42                     & 65.86                     \\
\multicolumn{1}{r|}{$(viii)$ Cross Entropy + Full ICNN}                           & ConvNet                     & 48.62                     & 68.12                     \\ \midrule
\multicolumn{1}{r|}{$(i)$ ICNN in Support}                                     & ResNet-12                   & 56.07                     & 77.44                     \\
\multicolumn{1}{r|}{$(ii)$ ICNN in Support \& Query}                             & ResNet-12                   & 58.93                     & 78.42                     \\
\multicolumn{1}{r|}{$(iii)$ CrossEntropy + ICNN in Support}                      & ResNet-12                   & 60.55                     & 79.11                     \\
\multicolumn{1}{r|}{$(iv)$ CrossEntropy + ICNN in Support \& Query}             & ResNet-12                   & 60.50                     & 79.58                     \\
\multicolumn{1}{r|}{$(v)$ ICNN in Support + Query(Prototypes)}                 & ResNet-12                   & 59.07                     & 79.18                     \\
\multicolumn{1}{r|}{$(vi)$ CrossEntropy + ICNN in Support\&Query(Protos)} & ResNet-12                   & \textbf{60.79}            & \textbf{80.41}            \\
\multicolumn{1}{r|}{$(vii)$ Full ICNN}                                           & ResNet-12                   & 60.37                     & 78.79                     \\
\multicolumn{1}{r|}{$(viii)$ Cross Entropy + Full ICNN}                           & ResNet-12                   & 60.32                     & 79.30                     \\ \bottomrule
\end{tabular}
\caption{Design choices for the ICNN Loss function using a ConvNet and ResNet-12 as feature extractors and 5-way tasks on both 1-shot and 5-shot settings.}
\label{tab:icnn}
\end{table*}

\begin{table*}
\centering
\begin{tabular}{@{}rccc@{}}
\toprule
\textbf{Model}                                                 & \multicolumn{1}{l}{\textbf{Backbone}} & \textbf{1-shot}           & \textbf{5-shot}           \\ \midrule
\multicolumn{1}{r|}{Proto-Triplet}                             & ConvNet                               & 48.85                     & 67.79                     \\
\multicolumn{1}{r|}{Cross Entropy + Proto-Triplet}             & ConvNet                               & 41.66                     & 66.09                     \\
\multicolumn{1}{r|}{Proto-Triplet + Full ICNN}                 & ConvNet                               & 46.00                     & 62.72                     \\
\multicolumn{1}{r|}{Cross-Entropy + Proto-Triplet + Full ICNN} & ConvNet                               & \textbf{49.82}            & \textbf{68.76}            \\ \midrule
\multicolumn{1}{r|}{Proto-Triplet}                             & ResNet-12                             & \multicolumn{1}{l}{60.87} & \multicolumn{1}{l}{78.78} \\
\multicolumn{1}{r|}{Cross Entropy + Proto-Triplet}             & ResNet-12                             & \multicolumn{1}{l}{60.09} & \multicolumn{1}{l}{79.33} \\
\multicolumn{1}{r|}{Proto-Triplet + Full ICNN}                 & ResNet-12                             & \multicolumn{1}{l}{59.58} & \multicolumn{1}{l}{78.53} \\
\multicolumn{1}{r|}{Cross-Entropy + Proto-Triplet + Full ICNN} & ResNet-12 & \multicolumn{1}{l}{\textbf{61.32}} & \multicolumn{1}{l}{\textbf{79.93}} \\ \bottomrule
\end{tabular}
\caption{Design choices for the Proto-Triplet Loss function using a ConvNet and ResNet-12 as feature extractors and 5-way tasks on both 1-shot and 5-shot settings.}
\label{tab:prototriplet}
\end{table*}

\begin{figure*}[htp]
    \centering
    
    \begin{subfigure}[t]{0.98\textwidth}
        \centering
        \raisebox{-\height}{\includegraphics[width=0.7\textwidth]{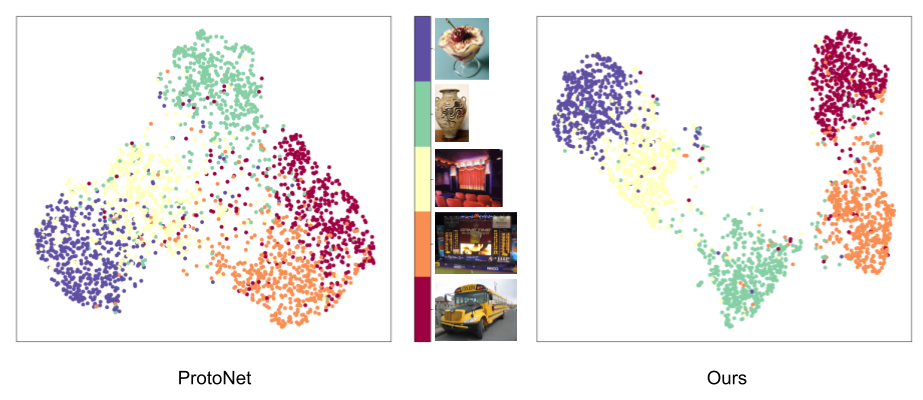}}
        \caption{MiniImagenet Dataset}
    \end{subfigure}
    
    \begin{subfigure}[t]{0.49\textwidth}
        \raisebox{-\height}{\includegraphics[width=\textwidth]{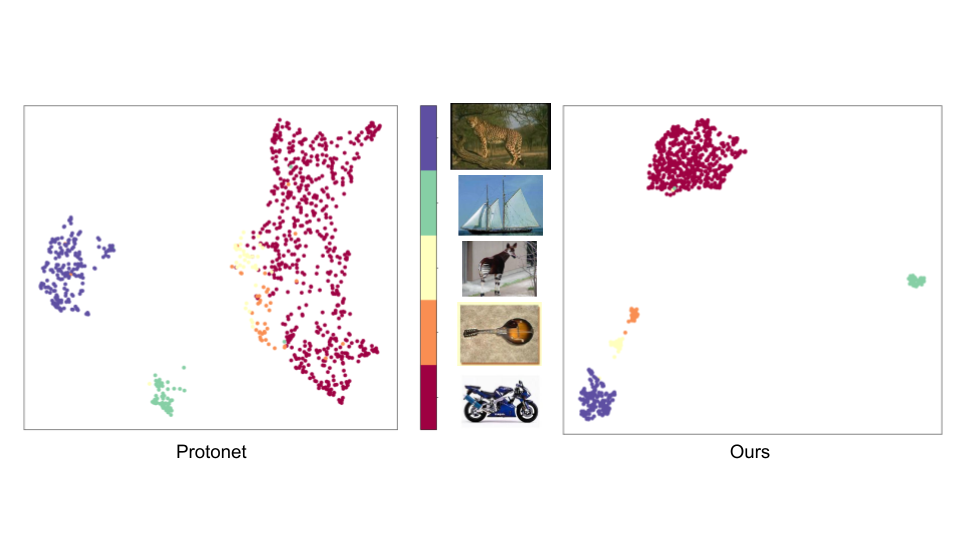}}
        \caption{Caltech Dataset}
    \end{subfigure}
    \hfill
    \begin{subfigure}[t]{0.49\textwidth}
        \raisebox{-\height}{\includegraphics[width=\textwidth]{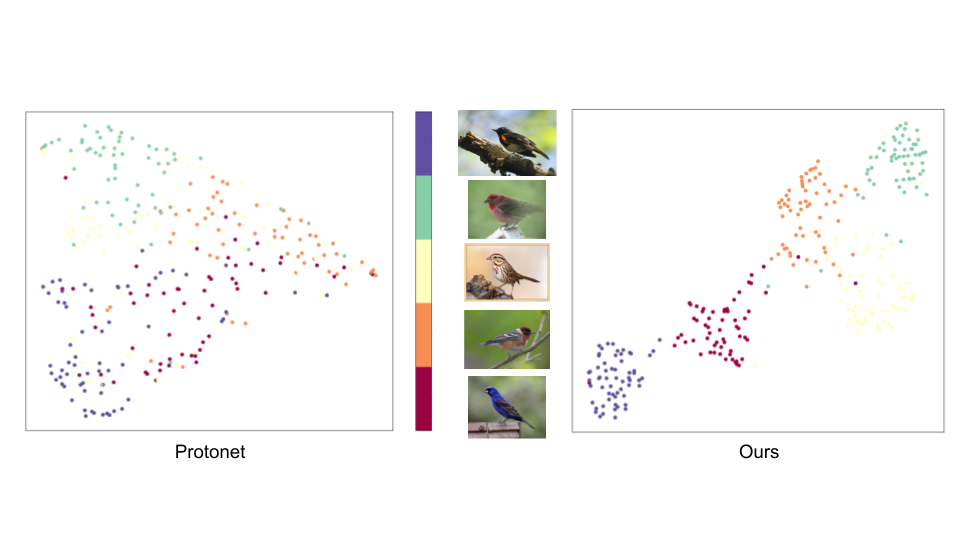}}
        \caption{CUB Dataset}
    \end{subfigure}
    
    \begin{subfigure}[t]{0.49\textwidth}
        \raisebox{-\height}{\includegraphics[width=\textwidth]{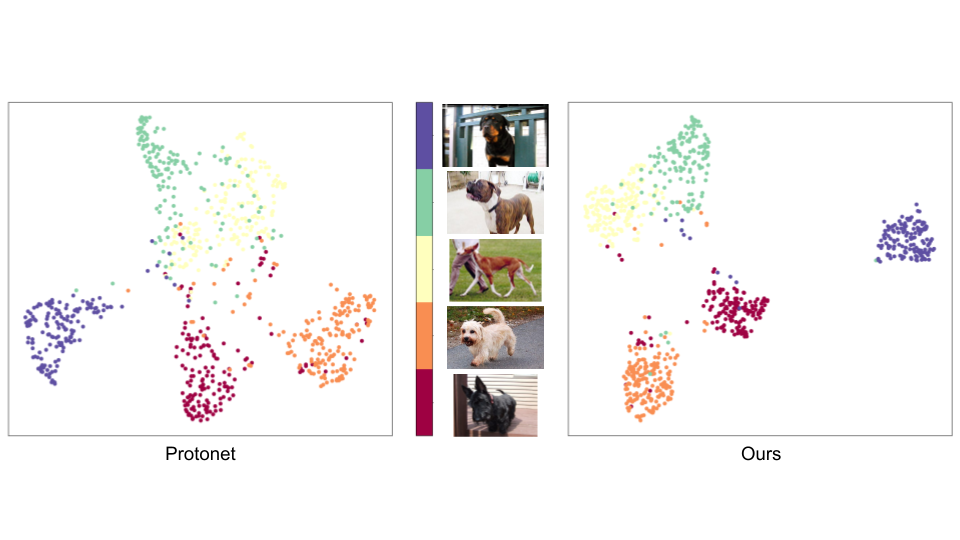}}
        \caption{Dogs Dataset}
    \end{subfigure}
    \hfill
    \begin{subfigure}[t]{0.49\textwidth}
        \raisebox{-\height}{\includegraphics[width=\textwidth]{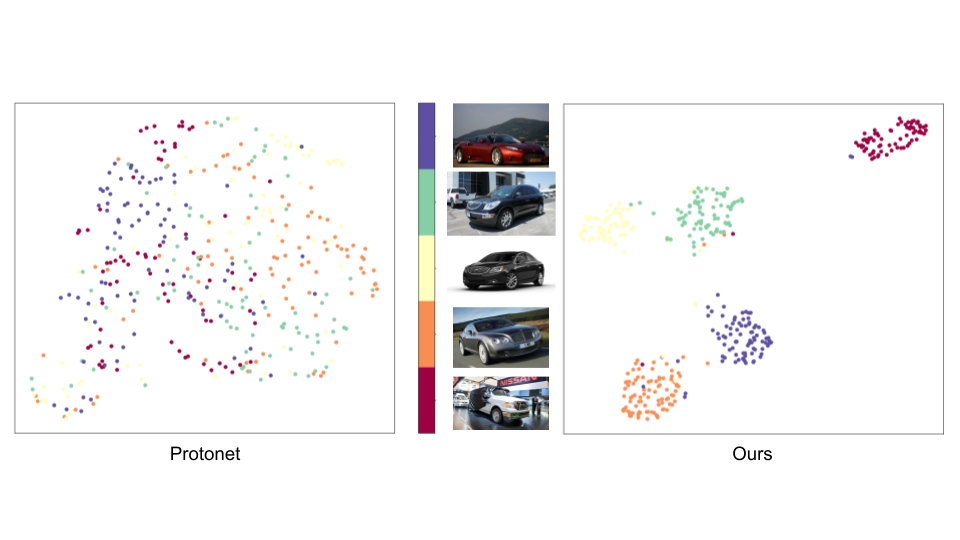}}
        \caption{Cars Dataset}
    \end{subfigure}
    
    \caption{Umap visualization of the improved feature embeddings obtained from the test set for each dataset. The feature embeddings were computed in a 5-way 5-shot setting by constructing 1,000 tasks from the test data. When creating the tasks for the feature visualization, we kept fixed 5 classes and randomly draw samples to construct the support and query sets}
    \label{fig:umaps}
\end{figure*}

\begin{table*}[]
\centering
\begin{tabular}{@{}ccccc@{}}
\toprule
\multicolumn{2}{c}{Dataset}                 & ProtoNet     & K-tuplets    & Ours                  \\ \midrule
\multirow{2}{*}{CUB-200}       & 5way-1shot & 39.39 ± 0.68 & 40.16 ± 0.68 & \textbf{67.56 ± 0.65} \\
                               & 5way-5shot & 56.06 ± 0.66 & 56.96 ± 0.65 & \textbf{83.62 ± 0.73} \\ \hline
\multirow{2}{*}{Caltech-101}   & 5way-1shot & 53.28 ± 0.78 & 61.00 ± 0.81 & \textbf{69.79 ± 0.57} \\
                               & 5way-5shot & 72.96 ± 0.67 & 75.60 ± 0.66 & \textbf{87.85 ± 0.96} \\ \hline
\multirow{2}{*}{Stanford Dogs} & 5way-1shot & 33.11 ± 0.64 & 37.33 ± 0.65 & \textbf{58.29 ± 0.28} \\
                               & 5way-5shot & 45.94 ± 0.65 & 49.97 ± 0.66 & \textbf{78.08 ± 0.21} \\ \hline
\multirow{2}{*}{Stanford Cars} & 5way-1shot & 29.10 ± 0.75 & 31.20 ± 0.58 & \textbf{72.76 ± 0.86} \\
                               & 5way-5shot & 38.12 ± 0.60 & 47.10 ± 0.62 & \textbf{87.67 ± 0.69} \\ \hline
\end{tabular}%
\caption{Average few-shot classification accuracies (\%) on other datasets. All the experiments are conducted with the same network for fair comparison.}
\label{tab:generalization}
\end{table*} 

\begin{figure*}[htp]
    \centering
    
    \begin{subfigure}[t]{0.49\textwidth}
        \raisebox{-\height}{\includegraphics[width=\textwidth]{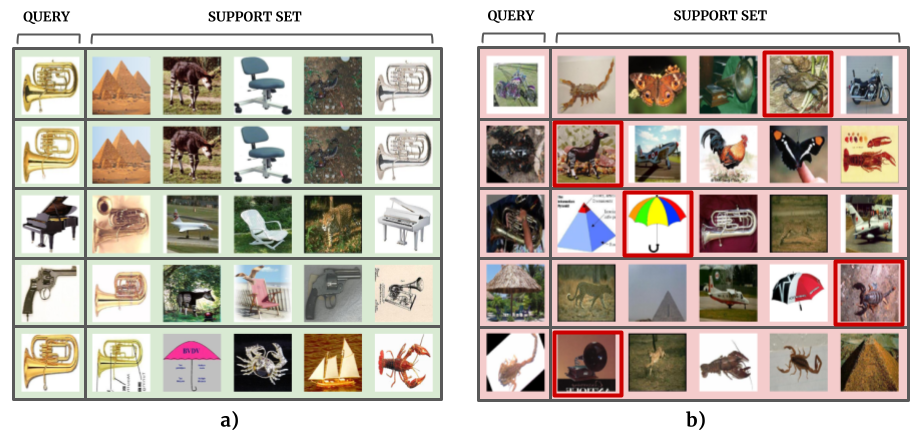}}
        \caption{Caltech Dataset}
    \end{subfigure}
    \hfill
    \begin{subfigure}[t]{0.49\textwidth}
        \raisebox{-\height}{\includegraphics[width=\textwidth]{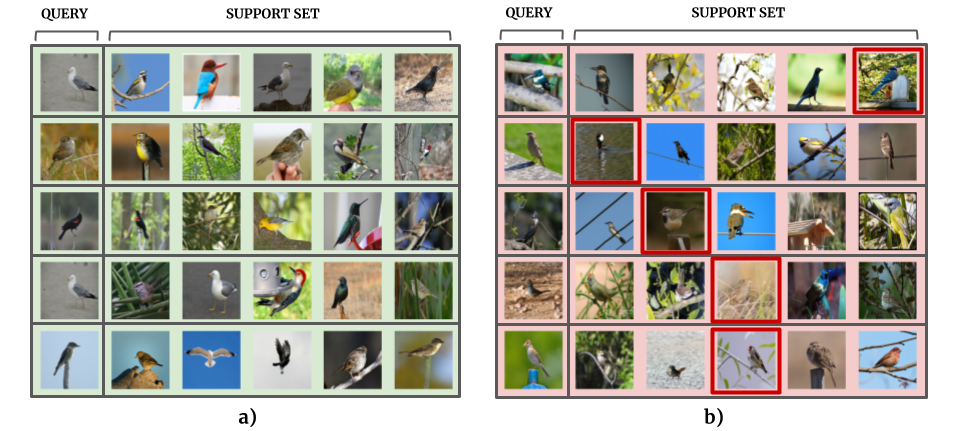}}
        \caption{CUB Dataset}
    \end{subfigure}
    
    \begin{subfigure}[t]{0.49\textwidth}
        \raisebox{-\height}{\includegraphics[width=\textwidth]{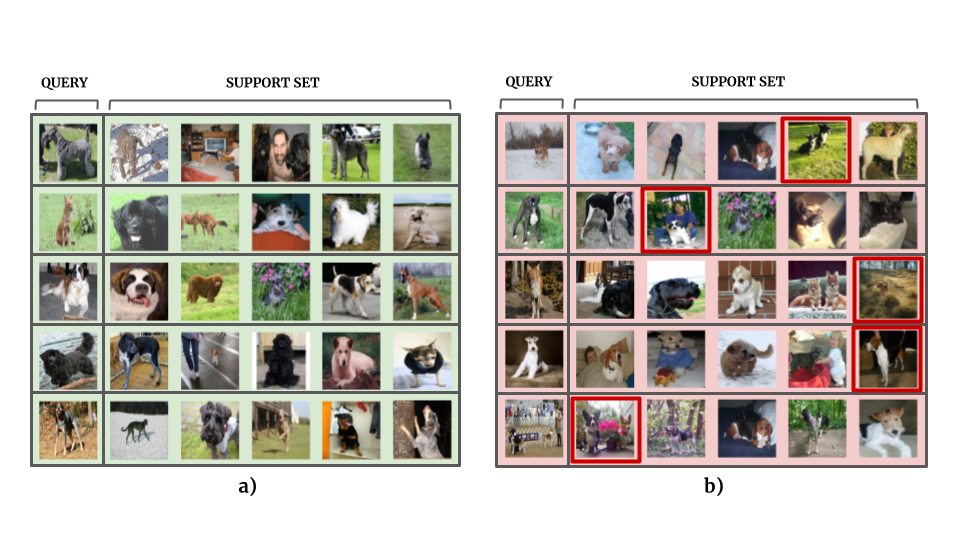}}
        \caption{Dogs Dataset}
    \end{subfigure}
    \hfill
    \begin{subfigure}[t]{0.49\textwidth}
        \raisebox{-\height}{\includegraphics[width=\textwidth]{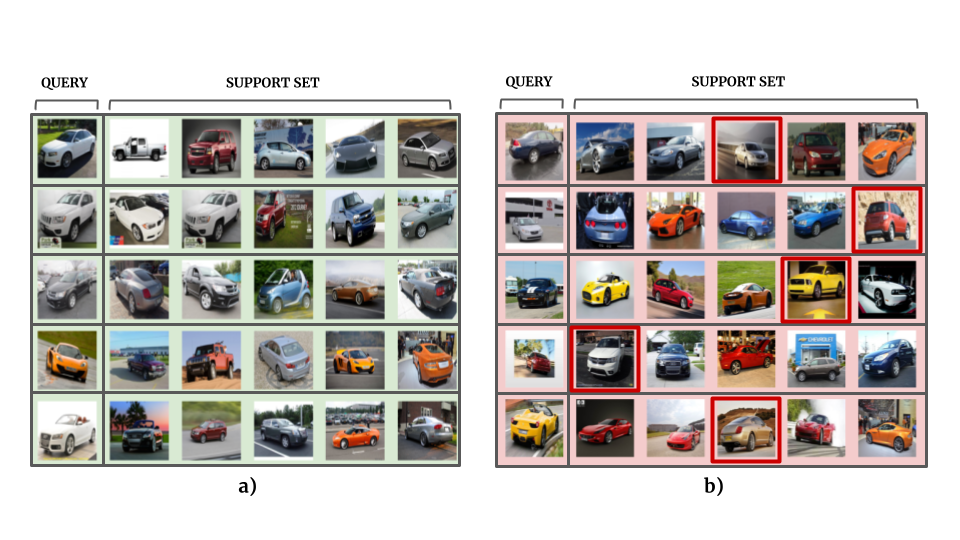}}
        \caption{Cars Dataset}
    \end{subfigure}
    
    \caption{Examples of tasks on the testing phase for different datasets. For each dataset, we report 5 tasks where the model classified correctly (green side) and 5 tasks where it failed to correctly classify and where it was very far from the correct result. }
    \label{fig:examples}
    
\end{figure*}

The results for these different ablation studies for the ICNN loss are reported in Table \ref{tab:icnn}. Our baseline is the simplest form of using the ICNN loss, by applying the score only to the support data. This gives us the worst results from the performed experiments, and are not competitive with state-of-the-art models. Then, we test our loss function by applying the score to different sets of data. For $(ii)$ and $(iv)$, we carried out tests using ICNN on support data and query data separately, with the difference that $(iv)$ uses the additional loss of the cross-entropy. When using a ConvNet, we obtained the best results by using the cross entropy and ICNN in support and query (iv), improving the baseline performance by around 8\%. For $(v)$ and $(vi)$ we modify the ICNN score on the query by comparing distances from each query to the prototypes instead of the comparing the distances from each query to the support data. When using the ResNet-12 as the feature extractor, this combination $(vi)$ gave us the best results, improving the baseline by around 4\% for the 1-shot setting and 3\% for the 5-shot setting. The last experiments, $(vii)$ and $(viii)$, apply the ICNN score to the combined set of support data and query data. These last experiments didn't lead us to the best results, but they are also quite competitive.

The results for the ablation studies for the proto-triplet loss are reported in Table \ref{tab:generalization}. We evaluate the combination of three loss functions: the cross-entropy, proto-triplet and the full version of the ICNN loss. First we test the proto-triplet loss alone to optimize the network, and obtain already a good accuracy performance compared with state-of-the-art results. Then, we continue by mixing the loss functions and we can observe that combining cross entropy with proto-triplet and ICNN loss with proto-triplet we get a worse performance than the obtained using only the proto-triplet loss. The last design choice is the combination of cross entropy, ICNN loss and proto-triplet loss, which is the one with the best accuracy performance by improving the result by 1\% using both feature extractors.


\subsection{Results of generalization on other datasets}

In order to measure the capabilities of generalization of the approach presented in this work, we evaluated the proposed losses and models on different datasets. Each dataset has distinct characteristics such as the number of classes, and the number of samples per class. The selected datasets are the following: CUB-200, Caltech-101 \cite{griffin2007caltech}, Stanford dogs, and Stanford cars \cite{stanford_cars}.

In order to compare our results with the state-of-the-art, we evaluated two backbones: ProtoNet and K-tuplets. For each backbone, 
two FSL configurations (5way-1shot, and 5way-5shot) were evaluated for each of the datasets (see Table \ref{tab:generalization}).

For visualization purposes and a qualitative evaluation of the results, we compared our model with those of ProtoNet using UMAP (Fig. \ref{fig:umaps}). 
In order to reinforce the quantitative (Fig. \ref{fig:umaps}) and qualitative (Table \ref{tab:generalization}) results, we show in Fig. \ref{fig:test}  examples of tasks on the testing phase for different datasets. For each dataset, we report 5 tasks where the model classified correctly (green side) and 5 tasks where it failed to correctly classify and where it was very far from the correct result.

\textbf{CUB dataset:} Our results, evaluated in CUB dataset have yielded an accuracy of 83.62\% for 5way-5shot congiguration (up to 26\% higher than K-tuplets and ProtoNet). Similarly, in 5way-1shot configuration, our approach surpasses the results of both models up to 27\% accuracy. In order to illustrate qualitatively the results in Fig. \ref{fig:umaps}c, we can see the effect of our model with respect to ProtoNet, which produces a sparse and overlapping point cloud. On the other hand, in our model, it can be observed that the groups of points for each class are well separated, although the points are slightly more dispersed.
The CUB dataset is the most difficult group of images to classify in our tests. This may be due to high intra-class similarities in the samples (birds with similar characteristics, Fig. \ref{fig:examples}c).  However, although it is difficult to classify this type of images, our approach has successfully reduced the intra-class distance and maximize the inter-class distances.

\textbf{Caltech dataset:} For the Caltech dataset, the results obtained for ProtoNet and K-tuplets using a 5way-1shot configuration have lower performances than our model. Although in a 5way-5shot configuration, the results for ProtoNet and K-tuplets are relatokively good (achieving performances close to 70\%). However,  our model obtains up to 87.85\%, surpassing the two compared models by up to 12\%. In the UMAP visualization, the results for our model are outstanding (Fig. \ref{fig:umaps}b): the point clouds are well separated by maximizing the inter-class distance, and it compacts points of the same classes. This allows us to observe in Fig. 6b compact and separated groups of points, which is desirable in classification problems.  
In Fig. \ref{fig:examples}a, we can observe that during the Caltech test our model is efficient when there is a strong similarity between the true class of the support set and the query set (green side). However, the model tends to reduce its classification capability when natural images are presented, i.e., images with non-uniform background or rotation effects (red side).

\textbf{Dogs dataset:} In the Dogs dataset the 5way-1shot configuration (58.29\%) shows an improvement over ProtoNet and K-tuplets, while it is only surpassed by the 5way-5shot configuration (78.08\%). In the UMAP visualization (Fig. \ref{fig:umaps}d) we can see how the points are better grouped when compared to the results produced by ProtoNet, which exhibit more dispersion. Although ProtoNet shows a good separation between classes, our model improves the inter-class distance.
For the Dogs dataset test (Fig. \ref{fig:examples}c), we can observe that even though this group of images contains non-homogeneous backgrounds, the model is able to discriminate between the different classes (green side). However, it can be observed that it is sensitive to objects that do not represent the class (red side), such as the presence of people or other objects in the test image.

\textbf{Cars dataset:} The Cars dataset has a larger scatter in the UMAP feature space for ProtoNet. As can be seen in Fig. \ref{fig:umaps}e (ProtoNet), the intra-class distance is relatively large, whereas inter-class distances are very small. The results achieved with ProtoNet go as high as 29.10\% and 38.14\% using 5-way 1-shot, and 5-way 5-shot, respectively. 
On the other hand, the results of our model are promising. Quantitatively, results of 72.76\% and 87.67\% are achieved for 5-way 1-shot, and 5-way 5-shot, respectively. Qualitatively, the feature space is considerably improved. In Fig. \ref{fig:umaps}e, tight clusters are observed for all individual classes, while the extra-class distance increases considerably.
Finally in the Cars dataset (Fig. \ref{fig:examples}d)we can observe a similar issue as with the CUB dataset. Despite having classes with very similar characteristics, this problem can be solved with models such as the one we propose in this paper. Maximizing the distance that exists between classes and reducing the intra-class distance, which allows us to discriminate even though the test images have highly similar distributions.

Overall, for each of the datasets , we have observed that our proposal outperforms both ProtoNet and K-tuplets using the two tested backbones, both quantitatively and qualitatively. The models trained using 5way-5shot present the best results for each of the datasets. This is because our models have more samples to learn during the FSL training process, yielding better representations. However, the 5way-1shot results are not entirely negligible. On the contrary, they extract features relevant for classification and are close to those obtained in 5way-5shot without the need to use a larger shot. In applications with few samples, where the number of images is limited, the use of techniques such as 5way-1shot for FSL could provide excellent results close to 5way-5shot. 



\section{Conclusion}

In this paper, we proposed two different loss functions to train an embedding network for few-shot image classification. These loss functions are based on the concepts of intra-class and inter-class distance, and have as the main objective to pull together instances of the same class and push further away the instances of different classes. The proposed model improves the accuracy performance when compared with baseline models which make use of metric learning approaches to solve the few-shot classification problem on the public benchmark dataset. The proposed model also obtains competitive results when compared with more recent methods which make use of a more robust embedding network, as we improve the accuracy for the 5-way 5-shot setting. Our current framework can be extended in several ways. For instance, we could make the hyper-parameters of the ICNN algorithm to be learnable instead of making them of a fixed value. Another direction is to test these loss functions with different metric meta-learning methods to see if they are allowing the network to learn better feature representations for few-shot tasks.





\section*{Acknowledgments}
The authors wish to thank the AI Hub and the Centro de Innovación de Internet de las Cosas at Tecnológico de Monterrey for their support for carrying the experiments in this paper in their NVIDIA's DGX computer.

This work has been supported by Azure Sponsorship credits granted by Microsoft's AI for Good Research Lab through the AI for Health program.

\bibliographystyle{cas-model2-names}

\bibliography{mybibfile}

\appendix
\section{Appendix}
\subsection{Proof Proposition 1}
\label{Proof:Proposition1}
\noindent
In order to prove the previous hypothesis (Propositions \ref{Proposition1}), the loss function ICNN (\ref{ICNNLOSS:def}) for $X$\footnote{In the case of $\left| X \right| = 2$, we have that the second term becomes the soft approximation of the rectifier relu function.}:
\begin{small}
\begin{equation}
ICNNLoss\left(X\right)=-\log\left[\frac{1}{\left|X\right|}\sum_{x_{i}\in X}\lambda\left(x_{i}\right)^{\frac{1}{p}}\omega\left(x_{i}\right)^{\frac{1}{q}}\gamma\left(x_{i}\right)^{\frac{1}{r}}\right] \nonumber
\end{equation}
\end{small}

Thus, without losing generality at each step of the gradient descent, there are lower bounds for $\omega$ and $\gamma$,   $\sigma_{\omega} \leq \omega\left(x_{i}\right)^{\frac{1}{q}}$ and $\sigma_{\gamma} \leq \gamma\left(x_{i}\right)^{\frac{1}{r}}$. Thus, we can obtain the following inequality,
\begin{small}
\begin{align}
    - ICNNLoss\left(X\right)  \geq  &\log\left[\tau\sum_{x_{i}\in X}\lambda\left(x_{i}\right)^{\frac{1}{p}}\right]  \nonumber
\end{align}
\end{small}
with 
\begin{small}
\begin{equation}
    \tau =	\frac{\sigma_{\omega}\times\sigma_{\gamma}}{\left|X\right|}, \nonumber
\end{equation}
\end{small}
we have that
\begin{small}
\begin{equation*}
    ICNNLoss\left(X\right) \leq -\log\left[\sum_{x_{i}\in X} \tau \lambda\left(x_{i}\right)^{\frac{1}{p}}\right]
\end{equation*}
\end{small}
Thus, the minimization of $-\log\left[\sum_{x_{i}\in X}\lambda\left(x_{i}\right)^{\frac{1}{p}}\right]$ implies the maximization of $\sum_{x_{i}\in X}\lambda\left(x_{i}\right)^{\frac{1}{p}}$. 
Now, it is known that under the gradient descent, and sample $\theta_{t+1}$ generated from $\theta_t$, the loss function $L$:
\begin{small}
\begin{equation*}
    L\left( \theta_{t+1} \right) \leq L\left( \theta_{t} \right).
\end{equation*}
\end{small}
Then, at each time $t$, $-\log\left[\sum_{x_{i}\in X}\lambda\left(x_{i} \vert \theta_t \right)^{\frac{1}{p}}\right]$ is smaller than the previous iteration. Therefore, intra-class is minimized and inter-class is maximized QED.

\subsection{Proof of Proposition 2}
\label{Proof:Proposition2}
(Propositions \ref{Proposition2})
Given the fact when $ICNNLoss\left(X\right)$ is minimized, all building elements are maximized. In particular, in the case of 
\begin{small}
\begin{equation*}
        \omega \left( x_i \right) = \alpha - \left[ Var\left( \lambda_{\tilde{x}} \left( x_i \right)  \right) +  Var\left( \lambda_{x}\left( x_i \right)  \right) \right],
\end{equation*}
\end{small}
the variances $Var\left( \lambda_{\tilde{x}} \left( x_i \right)  \right)$ and $Var\left( \lambda_{x}\left( x_i \right)  \right)$ are bounded by the $\alpha = k_1^2+k_2^2$ term. Thus, using a similar setup than the previous proof (Proposition \ref{Proof:Proposition1}) we can say that at each step $t$ of the gradient descent  the $Var\left( \lambda_{\tilde{x}} \left( x_i \right)  \right) + Var\left( \lambda_{x}\left( x_i \right)  \right) $ is minimized thus regularizing the  intra-class is minimized and inter-class distances QED.

\end{document}